\documentclass{article}
\usepackage{microtype}
\usepackage{graphicx}
\usepackage{subfigure}
\usepackage{booktabs} 

\usepackage{amssymb}
\newcommand{\PRD}{\text{PRD}}

\usepackage{hyperref}

\usepackage[accepted]{icml2019}

\icmltitlerunning{Revisiting Precision-Recall For Generative Models}

\usepackage[utf8]{inputenc}
\usepackage{amsmath}
\usepackage{amsthm}
\usepackage{dsfont}
\newtheorem{theorem}{Theorem}
\newtheorem{corollary}{Corollary}
\newtheorem{lemma}{Lemma}
\newtheorem{definition}{Definition}

\renewcommand{\Pr}{\mathbb{P}}
\newcommand{\R}{\mathbb{R}}
\newcommand{\1}{\mathds{1}}
\newcommand{\measSpace}{\Omega}
\newcommand{\measSet}{\mathcal M(\measSpace)}
\newcommand{\pmeasSet}{\mathcal M^+(\measSpace)}
\newcommand{\probaSet}{\mathcal M_p(\measSpace)}

\begin{document}
\twocolumn[
\icmltitle{
           Revisiting Precision and Recall Definition
           for Generative Model Evaluation}

\begin{icmlauthorlist}
\icmlauthor{Loïc Simon}{greyc}
\icmlauthor{Ryan Webster}{greyc}
\icmlauthor{Julien Rabin}{greyc}
\end{icmlauthorlist}

\icmlaffiliation{greyc}{Normandie Univ, UNICAEN, ENSICAEN, CNRS, GREYC}

\icmlcorrespondingauthor{Loïc Simon}{loic.simon@ensicaen.fr}

\icmlkeywords{Machine Learning, ICML}

\vskip 0.3in
]

\printAffiliationsAndNotice{} 

\begin{abstract}
In this article we revisit the definition of Precision-Recall (PR) curves for generative models proposed by ~\cite{sajjadi2018}.
Rather than providing a scalar for generative quality, PR curves distinguish mode-collapse (poor recall) and bad quality (poor precision). 
We first generalize their formulation to arbitrary measures, hence removing any restriction to finite support. 
We also expose a bridge between PR curves and 
type I and type II error rates
of likelihood ratio classifiers on the task of discriminating between samples of the two distributions.
Building upon this new perspective, we propose a novel algorithm to approximate precision-recall curves, that shares some interesting methodological properties with the hypothesis testing technique from~\cite{lopez2017twosamples}.
We demonstrate the interest of the proposed formulation over the original approach on controlled multi-modal datasets.
\end{abstract}

\section{Introduction}
\label{sec:intro}

This work addresses the question of the evaluation of generative models, such as Generative Adversarial Networks (GAN)~\cite{goodfellow2014generative} or Variational Auto-Encoders~\cite{kingma2013auto}, that have attracted a lot of attention in the last years.
These approaches aim at training a model to generate new samples from an unknown target distribution $P$, 
for which one has only access to a (sufficiently large) sample set $X_i \sim P$. 
While this class of methods have given state-of-the-art results in many applications (see \emph{e.g.} \cite{brock2018large} for image generation, \cite{iizuka2017globally} for inpainting, \emph{etc}), there is still a need for evaluation techniques than can automatically assess and compare the quality of generated samples $Y_i \sim Q$ from different models with the target distribution $P$, for which the likelihood $P(Y_i)$ is unknown.
Most of the time, such a comparison is just reduced to a simple visual inspection of the samples $Y_i$, but very recently several techniques have been proposed to address this problem that boils down to the comparison of two empirical distributions in high dimension. While generative models have seen successful applications far beyond just image data (such as speech enhancement~\cite{pascual2017segan}, text to image synthesis \cite{reed2016generative} or text translation \cite{lample2018word}), we will focus on image generation, as is popular 
in the recent literature. 

\begin{figure}[tb]
    \centering
    \includegraphics[width=\linewidth]{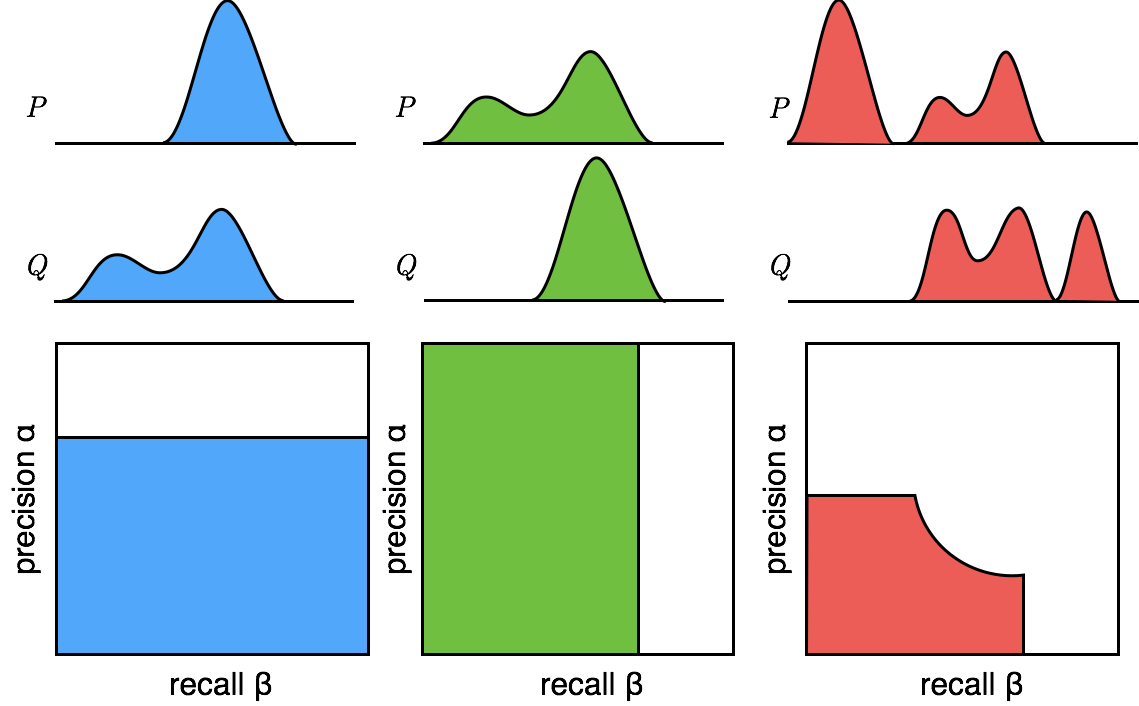}
    \caption{Illustration of precision-recall curves for multi-modal continuous distributions. Left: mode invention (precision is only partial but full recall). Middle: mode dropping (partial recall) but do not produce outliers (full precision but partial recall). Right: mode dropping / invention plus mode reweighting.}
    \label{fig:PRcurves}
\end{figure}

\paragraph{Previous Work}
When it comes to evaluating generative models of images, visual inspection, that is observing how ‘‘realistic’’ the images appear, remains the most important decider of the model's success. 
Indeed, state of the art methods, such a Progressive GANs~\cite{karras2018progressive} on face images or BigGAN~\cite{brock2018large} trained conditionally on ImageNet classes, include large grids of generated samples wherein the success of the method over previous approaches is visually obvious. 
Nonetheless, automatic evaluation of such models is extremely important, for example when conducting large scale empirical comparisons~\cite{LucicGANsEqual2017}, in cases where model failure is more subtle than simply poor image quality (\emph{e.g.} mode collapse) such as in~\cite{sajjadi2018}, or presumably in domains in which humans are less attuned to discern quality of samples.

Attempts to provide automatic assessment of image quality can be traced back to the first GAN methods \cite{RadfordDCGAN15}, where the authors assessed quality of generated samples with a nearest neighbor classifier. In \cite{SalimansIS16}, the so-called \textit{Inception Score} was introduced, which analyzes the entropy of image classes at the output of the Inception Network~\cite{szegedy2016rethinking}, which reflects if samples cover all classes and each clearly belongs to a particular class. In \cite{metz2017unrolled,webster2019GenOverfit} test set samples (\emph{i.e.} those unseen during training), are recovered via optimization. Successful generators are better at recovering all images from the training distribution, which in a controlled setting can be viewed as a notion of recall \cite{LucicGANsEqual2017, sajjadi2018}. 
In \cite{heusel2017gans}, the Fréchet Inception Distance was introduced (FID), which estimates the Fréchet distance between inception features of real and generated samples modeled as multivariate normal distributions. The FID has been widely adopted because of its consistency with human inspection and sensitivity to small changes in the real distribution (\emph{e.g.} slight blurring or small artifacts in generated images). A few recent approaches involve training a binary classifier to separate fake (\emph{i.e.} generated) samples $Y_i$ from real data samples $X_i$. 
In~\cite{lopez2017twosamples}, a score is defined from a two-sample statistical test of the hypothesis $P=Q$. 
Finally, in~\cite{im2018quantitatively}, classifiers trained with various divergences (normally used as objectives for discriminators during GAN training) are used to define a metric between $Q$ and $P$. 
Surprisingly, successful models such as WGAN~\cite{ArjovskyWGAN17} have the smallest distance even on those metrics which were not used for training (\emph{e.g.} a WGAN trained with the Wasserstein-1 distance evaluated with a least squares discriminator). 

Unfortunately as pointed out by \cite{sajjadi2018}, the popular FID only provides a scalar value that cannot distinguish a model $Q$ failing to cover all of $P$ (referred to henceforth as \emph{low recall}) from a model $Q$ which has poor sample quality (referred to as low precision).
For example, when modeling a distribution of face images, a $Q$ containing only male faces with high quality versus a $Q$ containing both genders with blurry faces may have equal FID.
Following the lead of~\cite{sajjadi2018}, we will consider another category where one wants not only to assess if the samples are of good quality (high precision) but also to measure if the generated distribution $Q$ captures the variability of the target one (high recall).
The reader may refer to Figure~\ref{fig:PRcurves} to gain a crude understanding of the intended purpose of precision and recall.
\cite{sajjadi2018} proposed an elegant definition of precision and recall for discrete distributions. They challenge their definition on image generation by discretizing the probability distributions $P$ and $Q$ over Inception features via K-means clustering.
Note that a similar notion was proposed by the authors of PACGAN \cite{lin2018pacgan} under the name of mode collapse region (denoted as $MCR(P,Q)$). Their motivation was to develop a theoretical tool to analyze how using multi-element samples in the discriminator can mitigate mode dropping. 

\paragraph{Contributions and outline}
The paper is organized as follows.
First, Section~\ref{sec:notations} recalls usual notations and some definitions from measure theory.
Then, we expose the main contributions of this paper: 
\begin{itemize}
    \item A first limit of~\cite{sajjadi2018} is the restriction to discrete probability distributions (\emph{i.e.} considering that samples live in a finite state space $\Omega$).
    In Section~\ref{sec:PR}, this assumption is dispensed by defining Precision-Recall curves from \emph{arbitrary} probability distributions for which some properties are then given;
    
    \item In the original work of~\cite{sajjadi2018} the Precision-Recall curves approach was opposed to the hypothesis testing techniques from~\cite{lopez2017twosamples}; 
    we demonstrate in Section~\ref{sec:classifier} that precision and recall are actually linear combinations of type I and type II errors of optimal likelihood ratio classifiers, and give as well some upper-bound guarantee for the estimation of Precision-Recall curves with non-optimal classifiers.
    Besides, our formulation also exhibits a relationship with the MCR notion proposed by \cite{lin2018pacgan} which turns out to be the ROC curves ($1-$type I versus type II errors) for optimal classifiers;
    
    \item
    Section~\ref{sec:algorithm} details the proposed algorithm to estimate Precision-Recall curves more accurately; 
    the clustering optimization step used in the original method is now simply replaced by the training of a classifier which learns to separate samples from the two datasets;
    
    \item The experimental Section~\ref{sec:exp} demonstrates the advantage of the proposed formulation in a controlled setting using labelled datasets (CIFAR10 and ImageNet categories), and then shows its practical interest for evaluating state-of-the art generative image models.
\end{itemize}

\section{Notions from standard measure theory}\label{sec:notations}

We start these notes by recalling some standard notations, definitions, and results of measure theory.
For the remainder, $(\measSpace,\mathcal A)$ represents a common measurable space, and we will denote $\measSet$ the set of signed measures, $\pmeasSet$ the set of positive measures and $\probaSet$ the set of probability distributions over that measurable space.
\begin{definition}
Let $\mu,\nu$ two signed measures. We denote by 
\begin{itemize}
    \item $\mathrm{supp}(\mu)$, the support of $\mu$;
    \item $\frac{d\mu}{d\nu}$, the Radon-Nykodim derivative of $\mu$ w.r.t. $\nu$;
    \item $|\mu|$, the total variation measure of  $\mu$;
    \item $\mu\wedge\nu = \min(\mu,\nu) := \frac 12 (\mu+\nu -|\mu-\nu|)$
    (a.k.a the measure of largest common mass between $\mu$ and $\nu$ \cite{piccoli2017}).
\end{itemize}
\end{definition}

The extended half real-line  is denoted by $\overline{\R^+} = \R^+ \cup \{\infty\}$.

\begin{theorem}[Hahn decomposition] Let $\mu\in\measSet$. Then there exists an essentially unique partition $\measSpace=\measSpace^+_{\mu}\sqcup \measSpace^-_{\mu}$ (\emph{i.e.} where $\measSpace^+_{\mu} \cap \measSpace^-_{\mu} = \varnothing$) 
such that $\forall A\in \mathcal A$:
\begin{equation*}
    \begin{split}
        A\subset\measSpace^+_{\mu} &\Rightarrow \mu(A)\geq 0\\
        A\subset\measSpace^-_{\mu} &\Rightarrow \mu(A)\leq 0\\
    \end{split}
    \,.
\end{equation*}
\end{theorem}

\begin{corollary}
\label{thm:hahn-cor}
Let $\mu,\nu\in\pmeasSet$. Then, $\forall A\in\mathcal A$, we have:
$$(\mu\wedge\nu)(A)= \mu(A\cap\measSpace^-_{\mu-\nu}) + \nu(A\cap\measSpace^+_{\mu-\nu})
.$$
\end{corollary}

\section{Precision-Recall set and curve}\label{sec:PR}

We follow \cite{sajjadi2018} for the definition of the Precision-Recall (PR) set that we extent to any arbitrary pair of probability distributions $P$ and $Q$, up to two additional minor changes. First, we have tried
to adapt their definition in a shorter form. Second, we include the left and lower boundaries in the PR set.
\begin{definition}
Let $P,Q$ two distributions from $\probaSet$. We refer to the Precision-Recall set $\PRD(P,Q)$ as the set of Precision-Recall pairs $(\alpha,\beta)\in\R^+\times\R^+$ such that
\begin{equation}
 \exists\, \mu \in\probaSet, P\geq\beta \mu , Q \geq \alpha \mu
 \;.
\end{equation}
\end{definition}
The \emph{precision} value $\alpha$ is related to the proportion of the generated distribution $Q$ that match the true data $P$, while conversely the \emph{recall} value $\beta$ is the amount of the distribution $P$ that can be reconstructed from $Q$.
Therefore, in the context of generative models, one would like to have admissible precision-recall pairs that are as close to $(1,1)$ as possible.
One can then easily show the following properties:
\begin{theorem}
Let $P,Q$ two distributions from $\probaSet$. Then,
\begin{enumerate}
    \item $(0,0)\in \PRD(P,Q)\subset [0,1]\times[0,1]$;
    \item $P=Q \Leftrightarrow (1,1)\in \PRD(P,Q)$;
    \item $(\alpha,\beta)\in \PRD(P,Q)$ and $\alpha'\leq\alpha, \beta'\leq\beta$ implies that $(\alpha',\beta')\in \PRD(P,Q)$.
\end{enumerate}
\end{theorem}

Because of the lack of natural order on $[0,1]\times[0,1]$, no point of $\PRD(P,Q)$ is strictly better than  all the others. Yet, the singular importance of $(1,1)$ should draw our attention to the Pareto front of $\PRD(P,Q)$ defined as follows. %
\begin{definition}
The precision recall-curve $\partial \PRD(P,Q)$ is the set of $(\alpha,\beta)\in \PRD(P,Q)$ such that
    \begin{equation*}
       \forall (\alpha',\beta')\in \PRD(P,Q), \alpha\geq\alpha'\text{ or }\beta\geq\beta' .
    \end{equation*}
\end{definition}

In fact, this frontier is a curve for which \cite{sajjadi2018} have exposed a parameterization. We generalize their result here (dropping any restriction to discrete probabilities).
\begin{theorem}
\label{thm:param}
Let $P,Q$ two distributions from $\probaSet$
and $(\alpha,\beta)$ positive. Then, denoting\footnote{As is conventionally surmised in measure theory $0\times\infty=0$ so that $\alpha_\infty=Q(\mathrm{supp}(P))$ and $\beta_0=P(\mathrm{supp}(Q))$.}
    \begin{equation}
    \forall \lambda \in \overline{\R^+},
    \left\{
    \begin{split}
        \alpha_\lambda :=& \left({(\lambda P)\wedge Q} \right) (\measSpace)\\
        \beta_\lambda :=& \left({P\wedge\tfrac 1\lambda Q}\right)(\measSpace)
    \end{split}
    \right.
    \end{equation}
\begin{enumerate}
    \item $(\alpha,\beta)\in \PRD(P,Q)$ iff $\alpha\leq \alpha_\lambda$ and $\beta \leq \beta_\lambda$ where $\lambda:=\frac \alpha\beta\in\overline{\R^+}$.
    \item  As a result,  the PR curve can be parameterized as:
\begin{equation}
    \partial \PRD(P,Q) = \lbrace (\alpha_\lambda, \beta_\lambda) / \lambda \in \overline{\R^{+}}\rbrace
    \, .
\end{equation}

\end{enumerate}
\end{theorem}

\begin{proof}
The second point derives easily from the first which we demonstrate now. 
Let $(\alpha,\beta)$ positive and $\lambda:=\frac \alpha\beta$. By definition $(\alpha, \beta)\in \PRD(P,Q)$ iff $\exists\ \mu\in\probaSet$
$$P\geq \beta \mu = \frac\alpha\lambda\mu\text{ and } Q \geq \alpha\mu$$
iff
$$\mu\leq \frac 1\alpha (\lambda P \wedge Q)(\measSpace) = \frac 1\beta (P \wedge \frac Q\lambda)$$
which yields the expected criteria given that $\mu(\measSpace)=1$.
\end{proof}

\section{Link with binary classification}\label{sec:classifier}
Let us consider samples $(X_i,Y_i)\sim P\times Q$ and as many Bernoulli variables $U_i\sim \mathcal B_{\frac 12}$. And let $Z_i=U_i X_i + (1-U_i) Y_i$. Then $Z_i\sim \Pr_Z$ follows a mixture of $P$ and $Q$, namely $\Pr_Z = \tfrac 12 (P + Q)$.
Then, let us consider the binary classification task where from $Z_i$, one should decide whether $U_i=1$ (often referred to as the null hypothesis).
We show that the precision-recall curve can be reinterpreted as  mixed error rates of binary classifiers obtained as likelihood ratio tests (hence the most powerful classifiers according to the celebrated Neyman-Pearson lemma).

\begin{theorem}\label{th:likelihood_ratio}
Let $\lambda\geq 0$. Let $Z=U X +(1-U)Y$ where $(X,Y,U)\sim P\times Q\times \mathcal B_{\frac 12}$. Defining the likelihood ratio classifier $\tilde U$ as the following indicator function
\begin{equation}\label{eq:optimal_classifier}
    \tilde U(Z) :=
    \displaystyle \1_{\lambda\frac{dP}{d\Pr_Z}(Z)\geq \frac{dQ}{d\Pr_Z}(Z)} \,,
\end{equation}
\begin{equation*}
    \text{then, } \quad
    \alpha_\lambda = \lambda \Pr(\tilde U=0|U=1)+ \Pr(\tilde U=1|U=0) \, .
\end{equation*}
\end{theorem}

\begin{proof}
Note that we can reformulate $\tilde U$ as 
$\tilde U(Z)=~\1_{\measSpace^+_{\lambda P-Q}}(Z)$.
Then, 
\begin{equation*}
\label{eq:alpha-classif}
\begin{split}
    \Pr(\tilde U=1|U=0)=&\int_\Omega \1_{\measSpace^+_{\lambda P-Q}}(z) d \Pr_Z(z|U=0)\\
    =& \int_\Omega \1_{\measSpace^+_{\lambda P-Q}}(z)dQ(z)
    = Q(\measSpace^+_{\lambda P-Q})
\end{split}\,.
\end{equation*}
Now, using 
$\1_{\tilde U=0} = \1_{\measSpace^-_{\lambda P-Q}}$, we have similarly $\Pr(\tilde U=0|U=1)=P(\measSpace^-_{\lambda P-Q})$.
%
Combining the two errors, we get
$$\lambda \Pr(\tilde U=0|U=1)+ \Pr(\tilde U=1|U=0) = (\lambda P\wedge Q)(\measSpace) = \alpha_\lambda$$
where we have used Corollary~\ref{thm:hahn-cor}.
\end{proof}

The previous protocol demonstrates that points on the PR curve are actually a linear combination of type I error rate (probability of rejection of the true null hypothesis $\Pr(\tilde U=0|U=1)$) with type II error rate ($\Pr(\tilde U=1|U=0)$).
It also shows that if one is able to compute the likelihood ratio classifier, then one could  virtually obtain the precision-recall curve $\partial \PRD(P,Q)$. 
Unfortunately, in practice the likelihoods  are unknown.
To alleviate this set-back, one can argue like \cite{menon2016linking} that optimizing standard classification losses is {\em in fine} equivalent to minimize a Bregman divergence to the likelihood ratio. Besides, we are going to show that using Eq.~\eqref{eq:alpha-classif} with any other classifier always yields an over-estimation of $\alpha_\lambda$ and $\beta_\lambda$. To do so, we will need the following lemma, which is merely a quantitative version of the Neyman-Pearson Lemma.

\begin{lemma}
\label{lemma-NP}
Let $\tilde U(Z)$ the likelihood ratio classifier defined in Eq.~\eqref{eq:optimal_classifier}, associated with the ratio $\lambda$. Then, any classifier $U'(Z)$ with a lower type II error, that is such that 
\begin{equation*}
    \Pr(U'=1|U=0) \le P(\tilde U=1|U=0) \, ,
\end{equation*}
undergoes an increase of the type I error such that
\begin{equation*}\left\{
\begin{array}{rl}
    \alpha_\lambda' :=\lambda P(U'=0|U=1) + P(U'=1|U=0) 
    & \ge \alpha_\lambda\\
    \beta_\lambda'  := P(U'=0|U=1) + \tfrac1\lambda P(U'=1|U=0) 
    & \ge \beta_\lambda\\
\end{array}\right.
\end{equation*}
\end{lemma}

\begin{proof}
The proof is similar to the classical proof of the Neyman-Pearson lemma (see Appendix~\ref{app:lemma-NP}). 
\end{proof}

\begin{theorem}\label{thm:optimality}
Let $\tilde U$ the likelihood ratio classifier from Eq.~\eqref{eq:optimal_classifier} associated with the ratio $\lambda$, and let $U'$ be any other classifier.
Using precision-recall pair $(\alpha_\lambda',\beta_\lambda')$ defined in Lemma~\ref{lemma-NP}, we have that
$$
\alpha_\lambda' \geq \alpha_\lambda
\text{ and }
\beta_\lambda' \geq \beta_\lambda
$$
\end{theorem}

\begin{proof}
The proof uses Lemma~\ref{lemma-NP} and its symmetric version (obtained by swapping the role of type-I and type-II errors). Three cases may arise:
\begin{enumerate}
    \item If $\Pr(U'=0|U=1)\geq \Pr(\tilde U=0|U=1)$ and $\Pr(U'=~1|U=0)\geq \Pr(\tilde U=1|U-0)$ then the conclusion of the theorem is trivially true;
    \item If $\Pr(U'=1|U=0)\leq \Pr(\tilde U=1|U=0)$, then the conclusion is ensured by Lemma~\ref{lemma-NP};
    \item If $\Pr(U'=0|U=1)\leq \Pr(\tilde U=0|U=1)$, then one should use the symmetric version of Lemma~\ref{lemma-NP}.
\end{enumerate}%
\end{proof}


\section{Algorithm}\label{sec:algorithm}

Based on the above analysis, we propose the Algorithm~\ref{algo:PRC} to estimate (via the function \texttt{estimatePRCurve}) the Precision-Recall curve of two probability distributions known through their respective sample sets.

\begin{algorithm}[!htb]
\SetKwInput{KwData}{Inputs}
\SetKwInput{KwResult}{Output}

 \SetAlgoLined\DontPrintSemicolon
 \SetKwFunction{algo}{estimatePRCurve}
 \SetKwFunction{testtrain}{createTrainTest}
 \SetKwFunction{learn}{learnClassifier}
 \SetKwFunction{eval}{estimatePRD}
 \SetKwFunction{sort}{sort}
 \KwData{
    Dataset of target/source sample pairs: $\mathcal D =\{(X_i,Y_i)\sim P\times Q\text{ i.i.d} /i\in\{1,\ldots,N\}\}$,\newline
    Parameterization of the PR curve: $\Lambda=\{\lambda_1,\ldots,\lambda_L\}$
 }
 \KwResult{$\partial \PRD_\Lambda\simeq\{(\alpha_\lambda,\beta_\lambda) / \lambda\in\Lambda\}$}
 \SetKwProg{myalg}{Algorithm}{}{}
 \SetKwProg{myproc}{Procedure}{}{}
 \myalg{\algo{$\mathcal D$, $\Lambda$}}{
 \nl $\mathcal D^{train}, \mathcal D^{test}=$\testtrain{$\mathcal D$}\;
 \nl $f = $\learn{$\mathcal D^{train}$}\;
 \nl $\partial \PRD_\Lambda= $\eval{$f$, $\mathcal D^{test}$, $\Lambda$}\;
 \nl \KwRet $\partial \PRD_\Lambda$\;
 }
 
 \;
 \setcounter{AlgoLine}{0}
 \myproc{\testtrain{$\mathcal D$}}{
 \nl $\mathcal D^{train} = \varnothing$, $\mathcal D^{test} = \varnothing$ \;
 \nl \For{$i\in\{1,\ldots, N\}$} {
    %
    \nl $U_i^{}\sim \mathcal{B}_{\frac 12}$\;
    \nl $Z_i^{train} = U_i^{} X_i + (1-U_i^{}) Y_i$ \;
    \nl $Z_i^{test} = (1 - U_i^{}) X_i + U_i^{} Y_i$ \;
    \nl $\mathcal D^{train} \leftarrow \mathcal D^{train}\cup \{(Z^{train}_i, U_i^{})\}$ \;
    \nl $\mathcal D^{test} \leftarrow \mathcal D^{test}\cup \{(Z^{test}_i, 1 - U_i^{})\}$ \;
 }
 \nl \KwRet $\mathcal D^{train}, \mathcal D^{test}$\;
 }
 
 \;
 \setcounter{AlgoLine}{0}
 \myproc{\eval{$f$, $\mathcal D^{test}$, $\Lambda$}}{
 \nl $fVals = \{f(z)/ (z,u)\in \mathcal D^{test}\}$\;
 \nl $errRates=\varnothing$ \;
 \nl $N_j = \left|{\{(z,u)\in\mathcal D^{test}/u= j \}}\right|$, for $j \in \{0,1\}$ \;
 \nl \For{$t\in fVals$} {
    %
    \nl \small $\mathrm{fpr}=\tfrac 1{N_1} \left|{\{(z,u)\in \mathcal D^{test}/ f(z)< t, u=1\}}\right|$ \;
    \nl \small $\mathrm{fnr} = \tfrac 1{N_0} \left|{\{(z,u)\in \mathcal D^{test}/ f(z)\geq t, u=0\}}\right|$ \;
    \nl $errRates \leftarrow errRates \cup \{(fpr, fnr)\}$ \;
 }
 \nl $\partial \PRD_\Lambda=\varnothing$\;
 \nl \For{$\lambda\in\Lambda$} {
    \nl $\alpha_\lambda=\min(\{\lambda\mathrm{fpr}+ \mathrm{fnr}/(\mathrm{fpr},\mathrm{fnr})\in {errRates}\})$\;
    \nl $\partial \PRD_\Lambda \leftarrow \partial \PRD_\Lambda\cup\{(\alpha_\lambda,\frac{\alpha_\lambda}\lambda)\}$\;
 }
 \nl \KwRet $\partial \PRD_\Lambda$\;
 }
 \;
 \caption{Classification-based estimation of the Precision-Recall curve.}
 \label{algo:PRC}
\end{algorithm}

\paragraph{Binary Classification}
We know from Theorem~\ref{th:likelihood_ratio} that the Precision-Recall curve can be exactly inferred from the likelihood ratio classifier denoted as $\tilde U$.
However, as explained earlier, since both the generated and target distributions ($Q$ and $P$ respectively) are unknown, one could not compute in practice this optimal classifier.
Instead, we propose to \emph{train} a binary classifier $U'$.
Recall that from Theorem~\ref{thm:optimality} the estimated PR curve, being computed with a sub-optimal classifier, lies therefore above the optimal one.
We only assume in the following that the classifier --denoted to as $f$ in the algorithm description-- which is returned by the function \texttt{learnClassifier} after the training, ranges in a continuous interval (e.g. $[0,1]$), so that the binary classifier $U'$ is actually obtained by thresholding: $U'(Z) = \1_{f(Z)\ge t}$.

As a result, since the classifier needs some training data, the $N$ sample pairs $\mathcal D =\{(X_i,Y_i), 1\le i \le N, X_i \sim P, Y_i \sim Q\}$ in the input dataset are first split into two sets $\mathcal D^{\text{train}}$ and $\mathcal D^{\text{test}}$ (function \texttt{createTrainTest}). For each image pair $(X_i,Y_i)$, a Bernoulli random variable $U_i$ with probability $\tfrac12$ is drawn to decide whether a true sample $X_i$ (when $U_i = 1$) or a fake one $Y_i$ (when $U_i=0$) is used for the training set $\mathcal D^{\text{train}}$. The other sample is then collected in the test set $\mathcal D^{\text{test}}$ to compute the PR curve.

\paragraph{Precision and Recall estimation}
Recall from Theorem~\ref{thm:param} that the PR curve $\partial \PRD = \{(\alpha_\lambda, \beta_\lambda), \lambda \in \overline{\R^+}\}$ is parametrized by the ratio $\lambda = \tfrac{\alpha}{\beta} $ between precision $\alpha$ and recall $\beta$.
We denote by $\partial \PRD_\Lambda$ the approximated PR curve when this parameter takes values in the set $\Lambda$.

Given a test dataset $\mathcal{D}^{\text{test}}$, the function \texttt{estimatePRD} computes the PR values $(\alpha_\lambda, \beta_\lambda)$ from the \emph{false positive rate} $\mathrm{fpr}$ and the \emph{false negative rate} $\mathrm{fnr}$ of the trained classifier $f$:
\begin{itemize}
    \item $\mathrm{fpr}$ corresponds to the empirical type I error rate, that is here (arbitrarily) the proportion of real samples $z = X_i$ (for which $u=1$) that are misclassified as generated samples (\emph{i.e.} when $f(z) < t$); 
    
    \item conversely, $\mathrm{fnr}$ is the empirical type II error rate, that is the proportion of generated samples $z = Y_i$ (for which $u=0$) that are misclassified as real samples (\emph{i.e.} $f(z) \ge t$);
\end{itemize}
Now, this raises the question of setting the threshold $t$ that defines the binary classifier $U'(z)=\1_{f(z)>t}$.
Since Theorem~\ref{thm:optimality} states that the computed 
precision and recall values $(\alpha_\lambda, \beta_\lambda)$ are actually upper-bound estimates, we use the minimum of these estimates when spanning the threshold value in the range of $f$.
Note that it is sufficient to consider the finite set $fVals$ of classification scores over  $\mathcal D^{test}$.

\paragraph{Comparison with ROC curves}
Using a ROC curve (for Receiver Operating Characteristic) to evaluate a binary classifier is very common in machine learning. 
Let us recall that it is the curve of the true positive rate ($1-\mathrm{fnr}$) against the false positive rate ($\mathrm{fpr}$) obtained for different classification thresholds.
Considering again the likelihood ratio test classifiers for all possible ratios would then provide the Pareto optimal ROC curve and could be used to assess if $P$ and $Q$ are similar or not. It turns out that the the frontier of the Mode Collapse Region proposed by \cite{lin2018pacgan} provides exactly this optimal ROC curve. For the recall, this notion is originally defined as follows:
\begin{equation*}
\begin{split}
MCR(P,Q)=\{&(\epsilon,\delta)/ 0\leq\epsilon<\delta\leq 1, \\
&\exists A\in \mathcal A, P(A)\geq\delta, Q(A)\leq\epsilon\}
\end{split}
\end{equation*}
From this definition, one can see that the MCR exhibits mode dropping by analyzing if part of the mass of $P$ is absent from $Q$. The notion differs from PRD at least in two ways. First MCR is not symmetric in $P$ and $Q$. Then it uses the mass of a subset $A$ instead of an auxiliary measure $\mu$ to characterize the shared / unshared mass between $P$ and $Q$. Despite those differences, the two notions serve a similar purpose. Given their respective interpretation as optimal type I vs type II errors, they mostly differ in terms of visual characterization of mode dropping.

\section{Experiments}\label{sec:exp}

In this section we demonstrate that Algorithm \ref{algo:PRC} is consistent with the expected notion of precision and recall on controlled datasets such as CIFAR-10 and Imagenet. The results even compare favorably to \cite{sajjadi2018} for such datasets. The situation is more complex when one distribution is made of generated samples, because the expected gold-standard precision-recall curve cannot be predicted in a trivial way.

In all our experiments, we compute the precision-recall curve between the distribution of features of the Inception Network~\cite{szegedy2016rethinking} (or some other network when specified) instead of using raw images (this choice will be discussed later on). In simple words, it means that we first extract inception features before training / evaluating the classifier. The classifier itself is an ensemble of $10$ linear classifiers. The consensus between the linear classifiers is computed by evaluating the median of their predictions. Besides, each linear classifier is trained independently with the ADAM algorithm. We progressively decrease the learning rate starting from $10^{-3}$ for $50$ epochs and use a fixed weight decay of $0.1$.
Any sophisticated classification method could be used to achieve our goal (deeper neural network, non-linear SVM, \emph{etc}), but this simplistic ensemble network turned out to be sufficient in practice.
Observe that this training procedure is replacing the pre-processing (K-means clustering) in the original approach of~\cite{sajjadi2018}, which relies also on inception features 
so that both methods share a similar time complexity.

\begin{figure}[tb]
    \centering
    \includegraphics[width=0.49\linewidth]{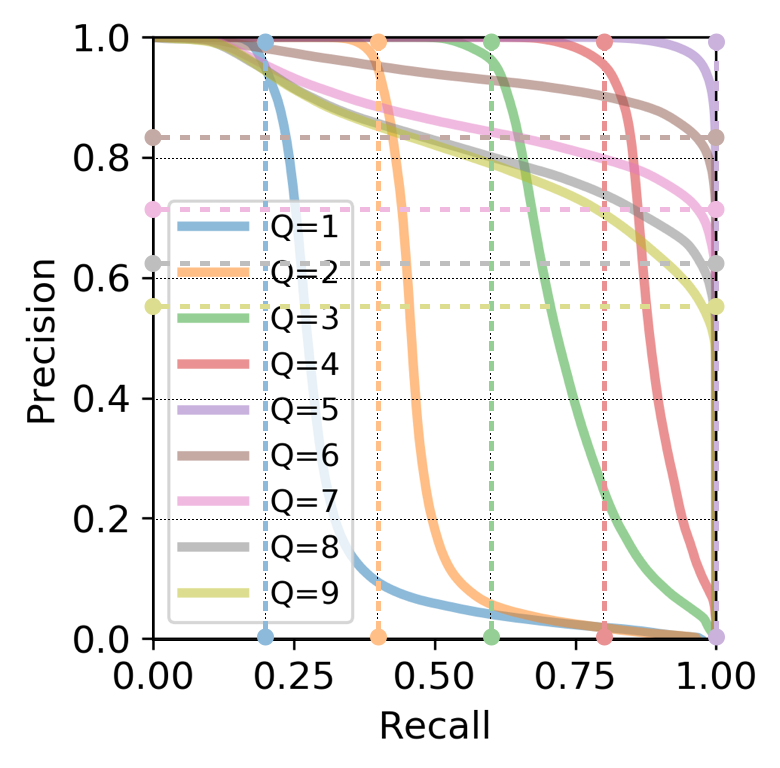}
    \includegraphics[width=0.49\linewidth]{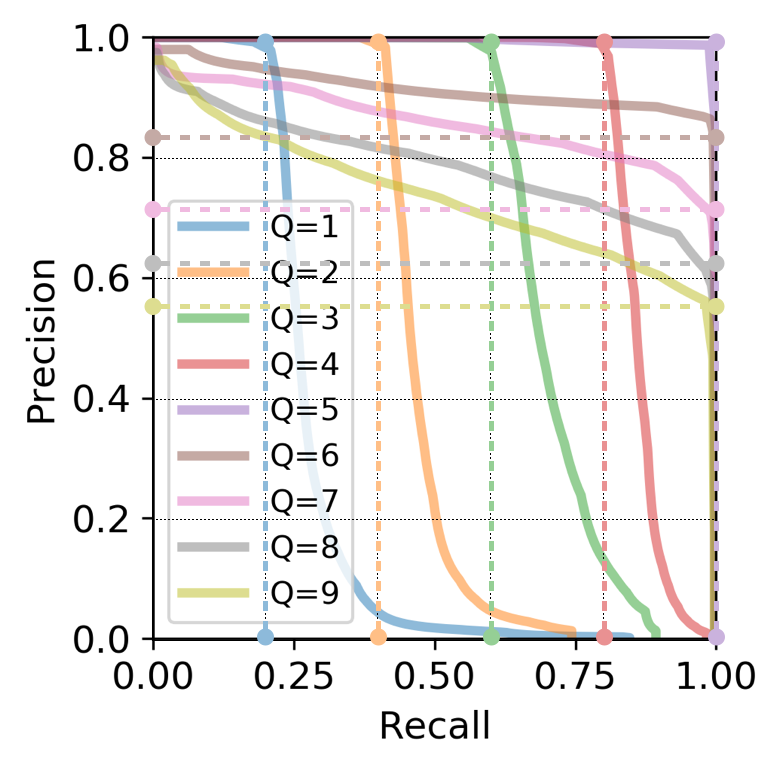}
    \caption{Precision-recall curves for $P$ made of the five first classes of CIFAR-10 versus $Q$ made of $q\in \{1,\ldots,9\}$ first classes. Left estimate from \cite{sajjadi2018} and right our implementation.}
    \label{fig:cifar10}
\end{figure}

Figure~\ref{fig:cifar10} reproduces an experiment proposed by~\cite{sajjadi2018}. It presents the estimated precision-recall curves on distributions made from CIFAR-10 samples. The reference distribution $P$ is always the same and it gathers samples from the first $5$ classes. On the other hand, $Q$ is composed of the first $q$ classes. When $q\leq 5$ we should expect a rectangular curve with a maximum precision of $1$ and maximum recall of $q/5$ (as illustrated in middle of Fig.~\ref{fig:PRcurves}). Similarly, when $q>5$ the expected curve is also rectangular one, but this time the maximum precision is $5/q$ and the maximum recall is $1$ (Fig.~\ref{fig:PRcurves}, left). These expected theoretical curves are shown in dash.
The original implementation from~\cite{sajjadi2018} is shown on the left and ours on the right. It is clear that both methods capture the intended behaviour of precision and recall.
Besides, two subtle differences can be observed. First, as implied by Theorem~\ref{thm:optimality} our implementation is always overestimating the theoretical curve (up to the variance due to finite samples). On the contrary, the clustering approach does not provide similar guarantee (as observed experimentally). Second, our implementation is slightly more accurate around the horizontal and vertical transitions.

One particular difficulty with the clustering approach lies in choosing the number of clusters.
While the original choice of $20$ is reasonable for simple distributions, it can fail to capture the complexity of strongly multi-modal distributions.
To highlight this phenomenon, we present in Figure~\ref{fig:imagenet} another controlled experiment with Imagenet samples.
In this case, $P$ and $Q$ are both composed of samples from $80$ classes, with a fixed ratio $\rho$ of common classes.
In this case, the expected curves can be predicted (see dash curves).
They correspond to rectangular curves with both maximal precision and recall equal to $\rho$.
As can be seen on the experimental curves, the clustering approach is prone to mixing the two datasets in the same clusters.
It therefore produces histograms that share a much heavier mass than the non discretized distributions, resulting in PR curves that depart strongly from the expected ones.
Of course, such a drawback could be partially fixed by adapting the number of clusters. However even then the clustering approach may fail, as is demonstrated in Figure~\ref{fig:varying-clusters}. In this experimient, the distribution $P$ is obtained approximately 60\% female faces and 40\% male faces from the CelebA dataset, while $Q$ is composed of female only. The theoretical curve is a sharp transition arising at recall $0.6$. This is well captured by our estimate (right curve) while varying the number of clusters always leads to an oversmooth estimate, with either under estimated precision or over-estimated recall.

\begin{figure}[tb]
    \centering
    \includegraphics[width=0.49\linewidth]{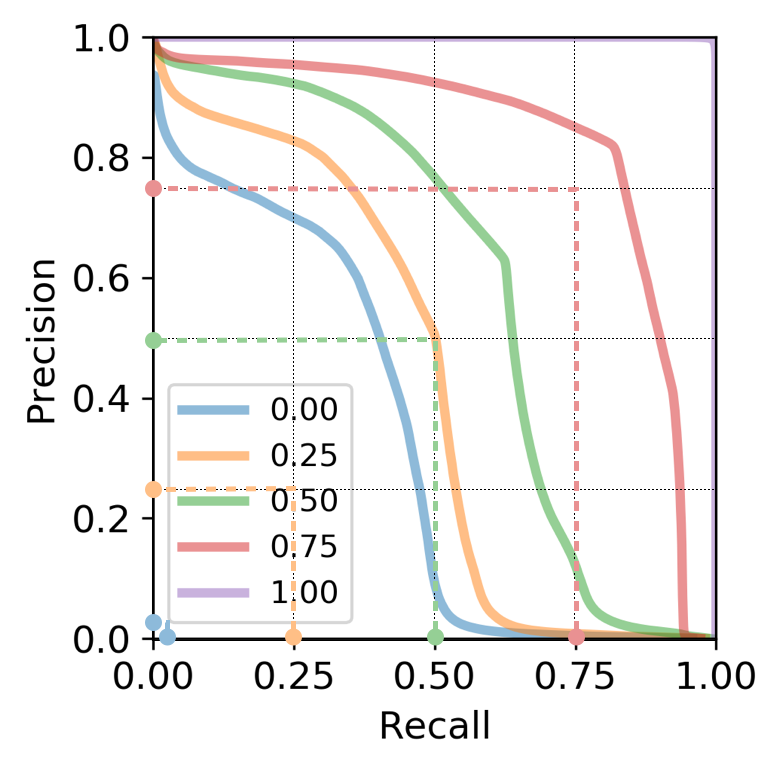}
    \includegraphics[width=0.49\linewidth]{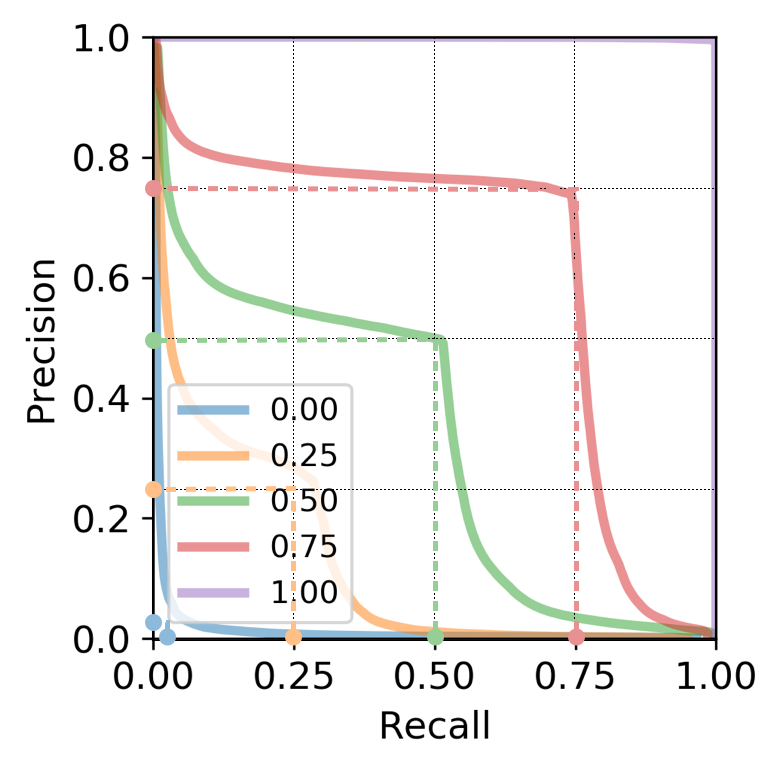}
    \caption{PR curves for $P$ and $Q$ made of 80 classes from ImageNet. The ratio of common classes varies from $0$ to $100\%$. Left: from \cite{sajjadi2018}. Right: our implementation.}
    \label{fig:imagenet}
\end{figure}

\paragraph{Experiments on Generated Images}
Figure~\ref{fig:GAN_PR} illustrates the proposed approach to three GANs trained on the Celeba-HQ~\cite{karras2018progressive} dataset. 
We highlight the two recent approaches of progressive GANs in~\cite{karras2018progressive} and the 0 centered gradient penalty ResNets found in~\cite{Mescheder2018ICML} as they produce realistic images. For comparison, we also include DCGAN~\cite{RadfordDCGAN15}. For analysis with Algorithm~\ref{algo:PRC}, we choose the first $N=1000$ images of Celeba-HQ, and generate as many images with each GAN. 
For training the classifiers, we split each set (real and fake images) into $900$ training images and $100$ test images. In light of the previous experiments, we choose to train our architecture on top of vision relevant features. 
Because we are dealing with faces, we choose the convolutional part of the VGG-Face network~\cite{Parkhi15VGGFace}. One advantage on using VGG-Face is that artifacts present in generated images, such as unrealistic backgrounds, are mitigated by the VGG-face network, so that classification can focus on the realism of facial features. Of course, small artifacts can be present in even high quality generators and a perfect classifier could "cheat" by only using seeing such artifacts. Fig.~\ref{fig:GAN_PR}, shows the computed PR curves for the three generators. Intuitively, networks with high precision should generate realistic images consistently. Progressive GANs achieve a maximum precision of $1.0$, and overall high precision, which is visually consistent. DCGAN is producing unrealistic images which is reflected by it's overall low precision. In some sense, recall reflects the diversity of the generated images with respect to the dataset and it is interesting to note all networks achieved higher recall than precision. 
Finally, for the sake of comparison with the FID~\cite{heusel2017gans}, the networks in Fig.~\ref{fig:GAN_PR} achieved FIDs of $25.23$, $27.61$ and $67.84$ respectively from left to right (lower is better). 
\begin{figure}[tb]
    \centering
    \includegraphics[width=0.49\linewidth]{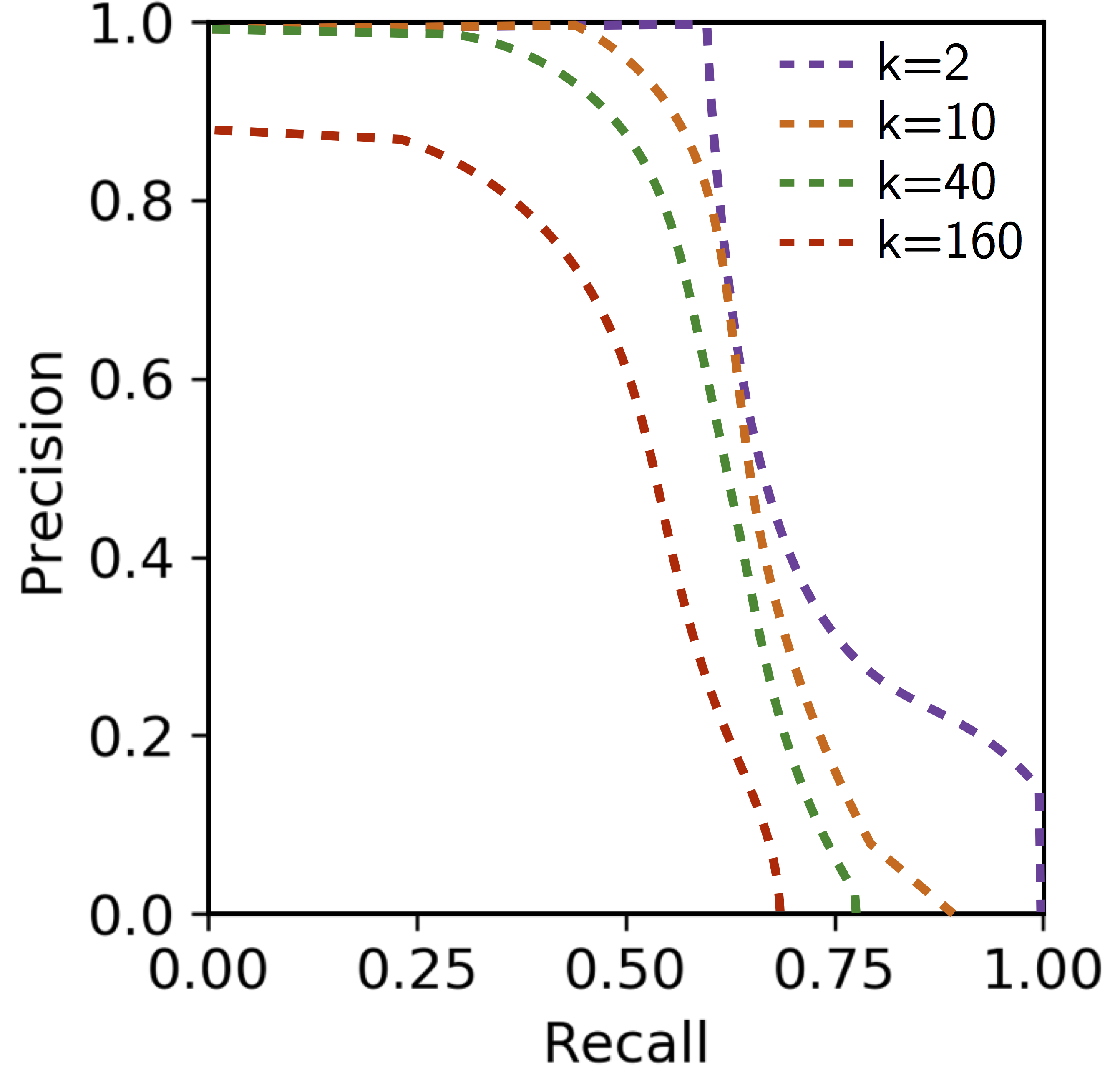}
    \includegraphics[width=0.49\linewidth]{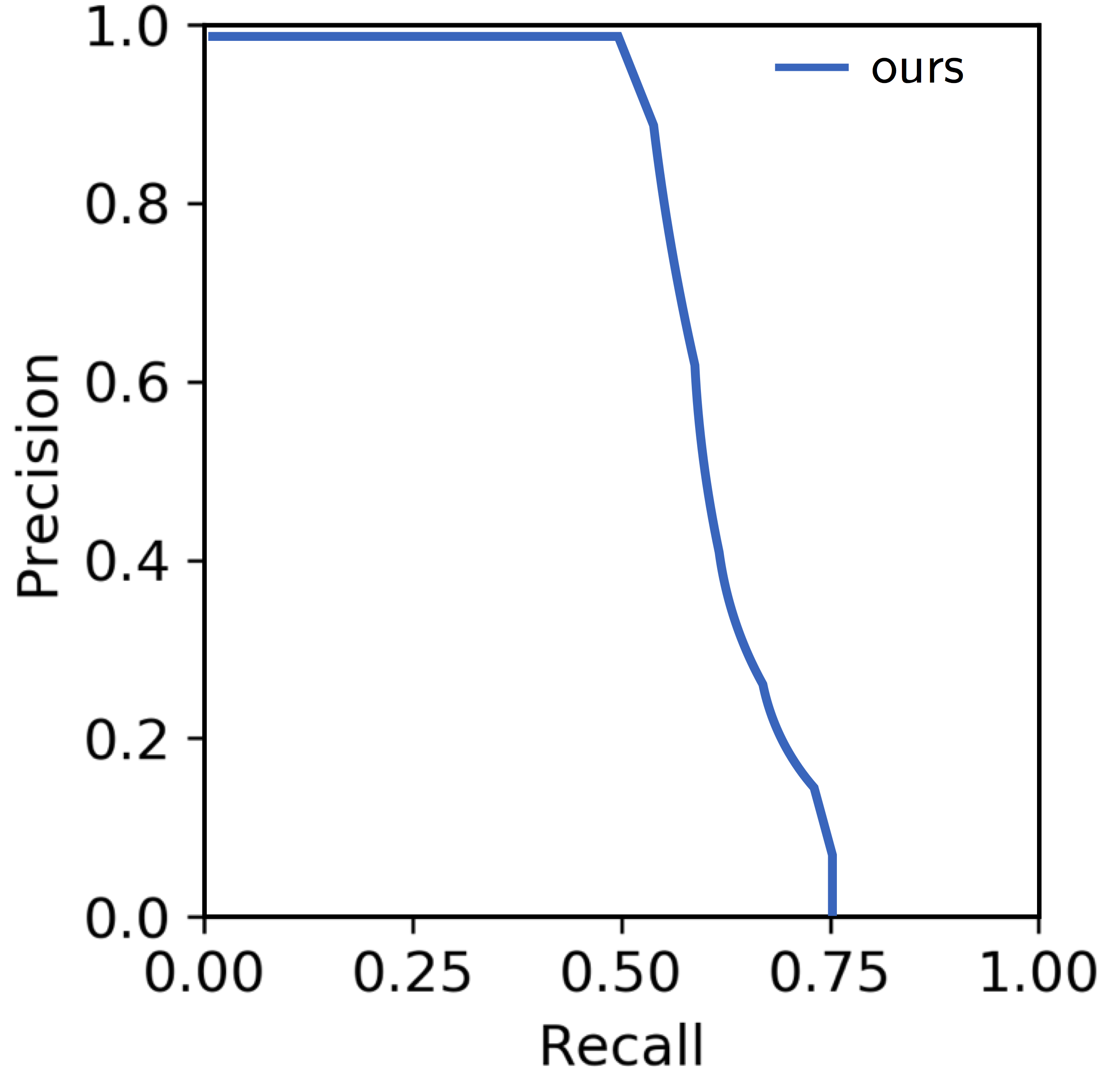}
    \caption{PR curves when $P$ is composed of faces from CelebA (60\% females) and $Q$ is composed of females only.  }
    \label{fig:varying-clusters}
\end{figure}
\paragraph{}
\iftrue
Next, we analyze BigGAN \cite{brock2018large} on ImageNet for our classification approach and the clustering approach presented in \cite{sajjadi2018}. Both approaches use inception features as before. We take $80$ images from the first $40$ classes of ImageNet, and then $20$ images from the first $40$ classes for test images. We use $20$ clusters for the K-means approach and a single linear layer for the classification approach. Fig.~\ref{fig:GAN_ImageNet} highlights a large difference between the approaches; the clustering approach overestimates the similarities between the distributions and the classification approach easily separates the two distributions. As was demonstrated in Fig.~\ref{fig:imagenet}, there are more classes than clusters, which could explain why images from both distributions may fall into the same clusters, in which ~\cite{sajjadi2018} will fail to discern the two distributions. It is interesting to note that the classifier easily separates the distributions despite the inception features being sparse for image samples. One can observe a lack of intra-class diversity in the BigGAN samples, which may be how the classifier discerns the samples. We leave further investigation of this discrepancy for future work. 
\fi
\begin{figure}[tb]
    \iftrue
        \centering
        \includegraphics[width=\linewidth,trim={22cm 0 0 0},clip]{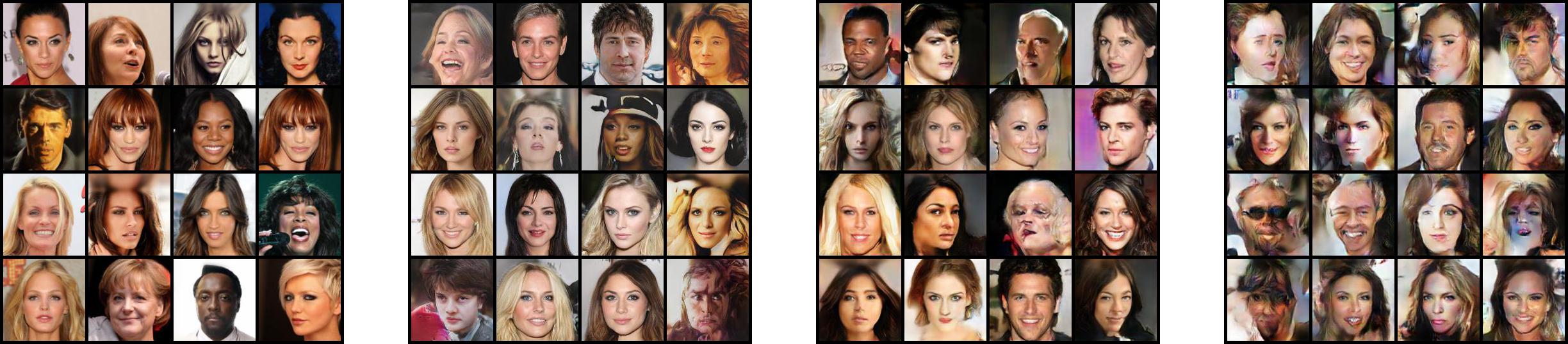}
        
        \centering
        \includegraphics[width=0.32\linewidth]{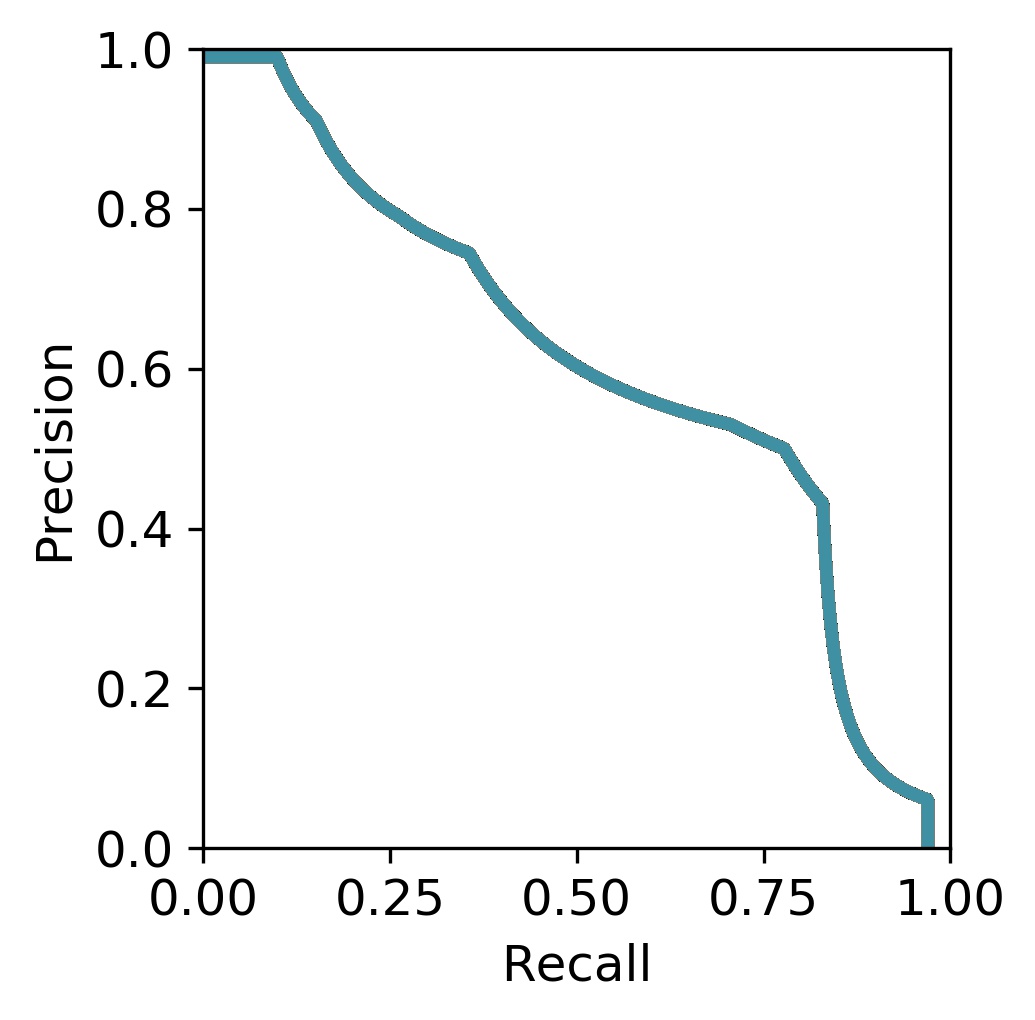}
        \hfill
        \includegraphics[width=0.32\linewidth]{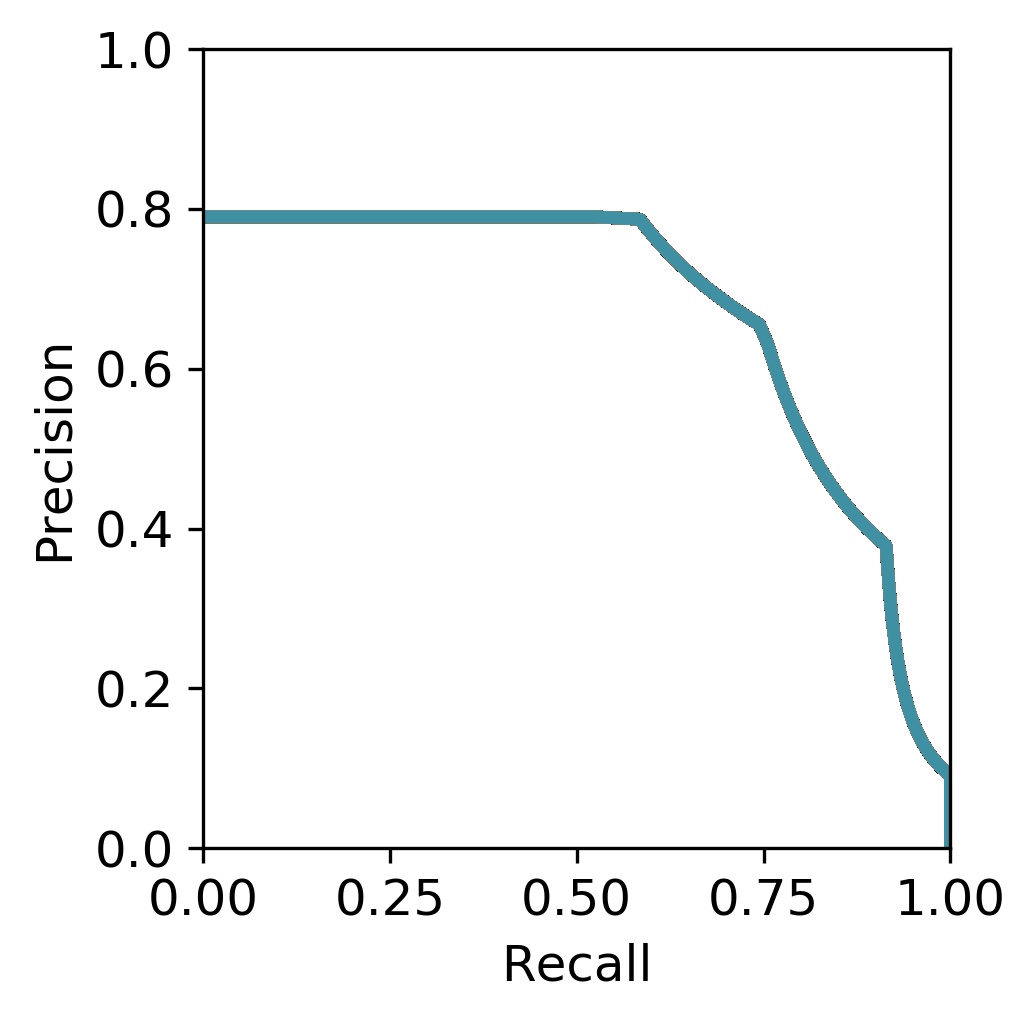}
        \hfill
        \includegraphics[width=0.32\linewidth]{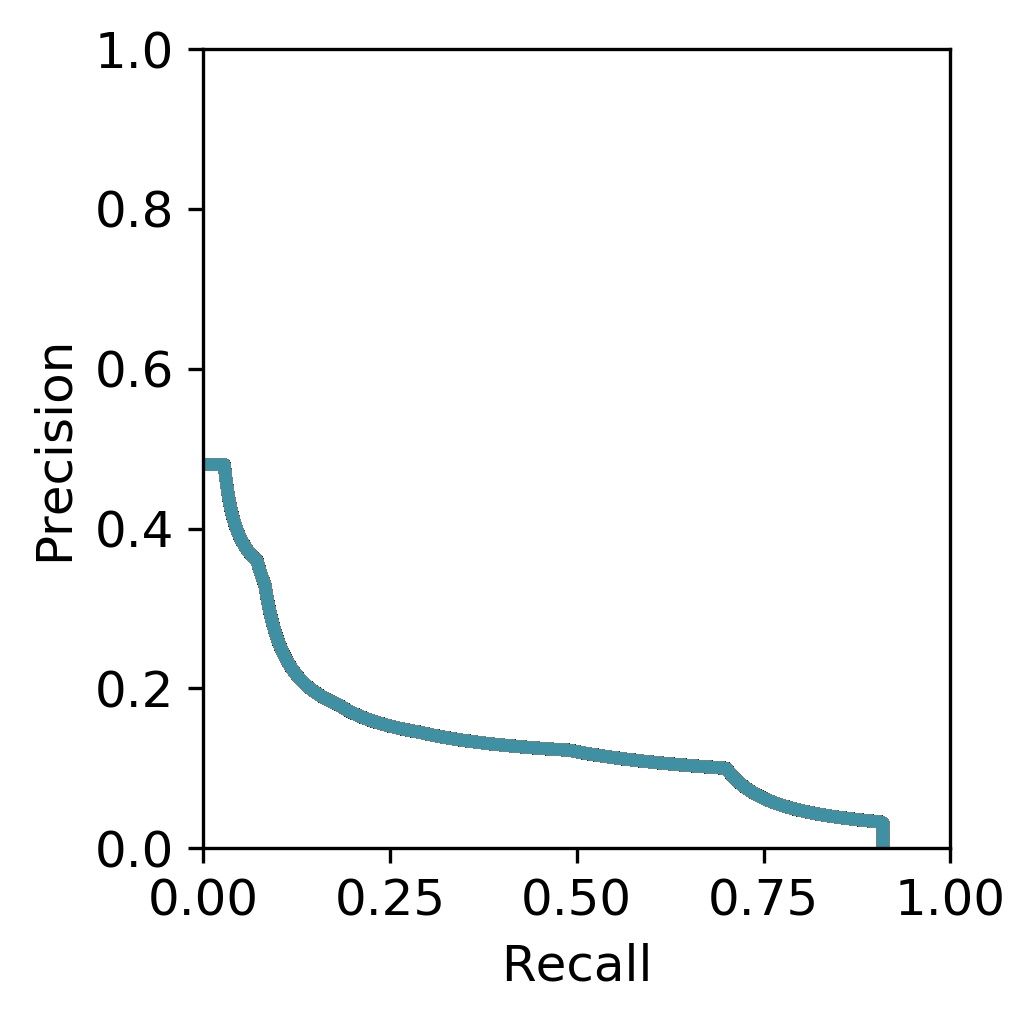}
    \else
        \centering
        \includegraphics[width=0.96\linewidth]{images/celeba_pggan_mesch_dcgan.jpg}
        
        \centering
        \hspace*{.24\linewidth}
        \includegraphics[width=0.24\linewidth]{images/celeba_pggan_vggf_classifier_PR.jpg}
        \includegraphics[width=0.24\linewidth]{images/celeba_mesch_vggf_classifier_PR.jpg}
        \includegraphics[width=0.24\linewidth]{images/celeba_dcgan_vggf_classifier_PR.jpg}
    \fi
    \caption{Precision-recall curves and generated images for various popular GANs on CelebA-HQ dataset \cite{karras2018progressive}.
    From left to right: 
    PGGAN~\cite{karras2018progressive},
    ResNet~\cite{Mescheder2018ICML}, and 
    DCGAN~\cite{RadfordDCGAN15}.
    }
    \label{fig:GAN_PR}
\end{figure}

\iftrue
\begin{figure}[tb]
    \centering
    \includegraphics[width=0.45\linewidth]{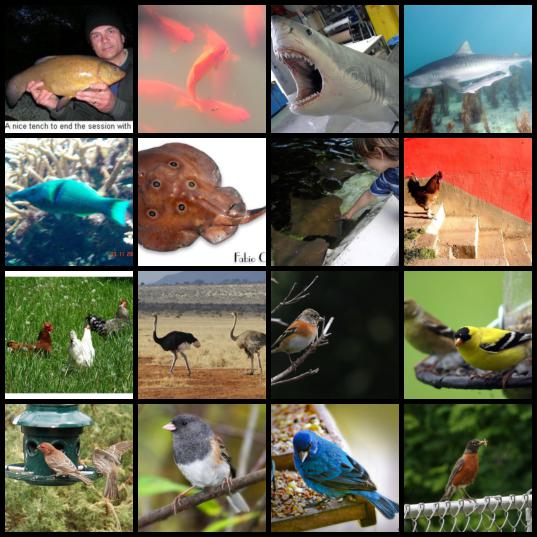}
    \hfill
    \includegraphics[width=0.45\linewidth]{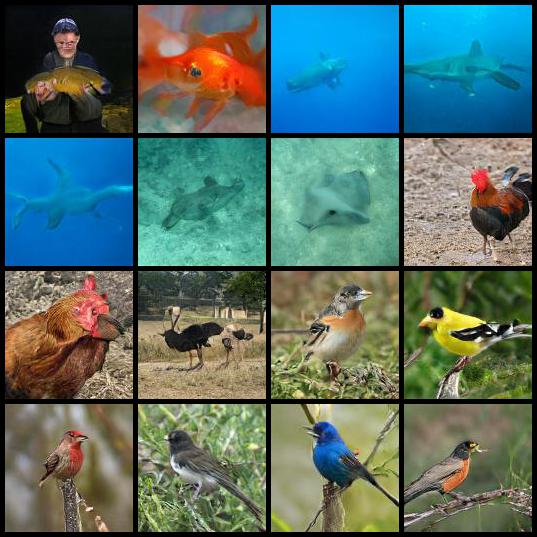}
    
    \centering
    \includegraphics[width=0.45\linewidth]{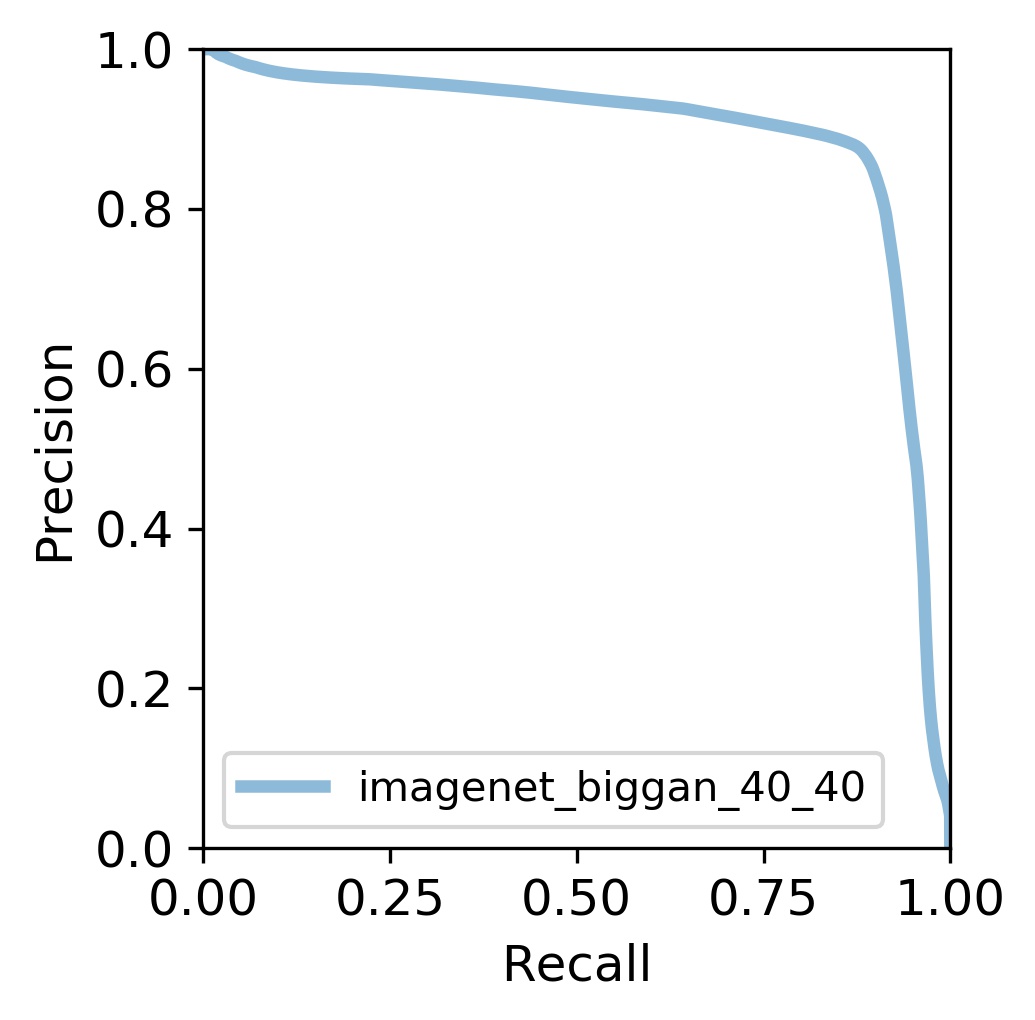}
    \hfill
    \includegraphics[width=0.45\linewidth]{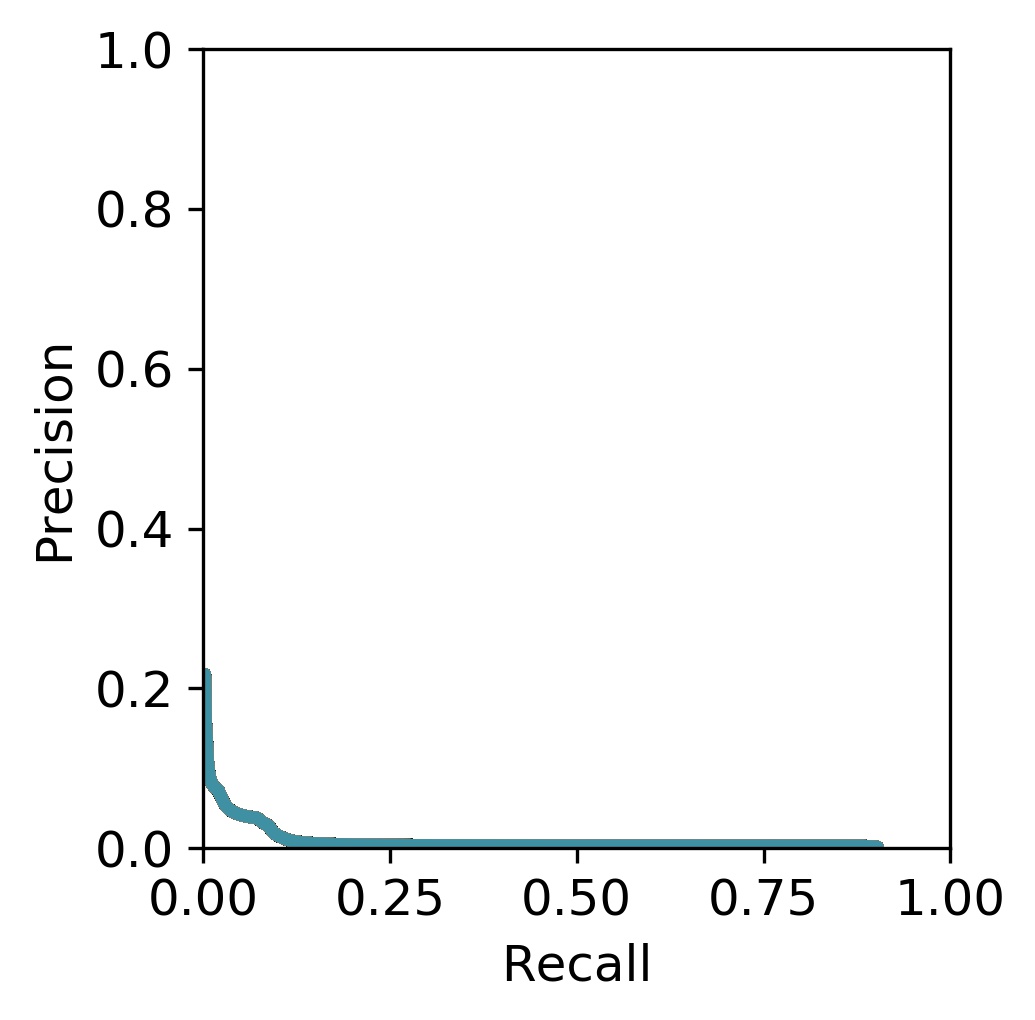}
    \caption{Evaluating generated samples on ImageNet.
        First row: Samples from various categories of ImageNet (on the left), and generated samples for the same categories from BigGAN~\cite{brock2018large} (on the right). 
        Second Row: PR curves computed with the clustering approach of~\cite{sajjadi2018} and ours.
        }
    \label{fig:GAN_ImageNet}
\end{figure}
\fi

\section{Discussion and future work}\label{sec:conclusion}

In this paper, we have revisited a recent definition of precision-recall curve for comparing two distributions. Besides extending precision and recall to arbitrary 
distributions, we have exhibited a dual perspective on such notions. In this new view, precision-recall curves are seen through the prism of binary classification. Our central result states that the Pareto optimal precision-recall pairs can be obtained as linear combinations of type I and type II errors of likelihood ratio classifiers. Last, we have provided a novel algorithm to evaluate the precision-recall curves from random samples drawn within the two involved distributions.

\paragraph{Discussion}
One achievement of our formulation is that one can directly define the precision-recall curves of distributions defined on continuous manifolds. In particular, our definition could be applied directly in the image domain, instead of first embedding the distribution in a feature space. 
From the strict computation perspective, there should not be any daunting obstacle in the way, as soon as we can have access to enough data to train a good classifier.
This is usually the case for generative models, since the standard datasets are quite massive.

However, it is not obvious whether classifiers trained on raw-data provide useful notions of PR-curves. 
Indeed, given the current state of affairs of generative modeling, we think that the raw image curves may be less useful.
Indeed, until now, even the best generative models produce artifacts (blurriness, structured noise, etc.).
As such, the theoretical distributions (real and generated) are mutually singular.
So, their theoretical precision-recall curve should be always trivial (\emph{i.e.} reduced to the origin).
It is hence a necessary evil to embed the distributions into a feature space as it allows a classifier to focus its attention on statistical disparities that are meaningful for the task at hand.
For instance, when evaluating a face generator, it makes sense to use features that are representative of facial attributes. Nonetheless, future work should investigate a wider variety of pre-trained features as well as classifiers trained on raw data to determine which method is most suitable for computing PR curves. 

\paragraph{Perspectives}
This work offers some interesting perspectives that we would like to investigate in the future.
First, as opposed to the usual GAN training procedure where a scalar divergence is used to assess the similarity between generated and target distributions, one could use the proposed precision and recall definitions to control the quality of the generator
while preventing mode-collapse.
For instance, the discriminator could use the role of the classifier, as it has been done in \cite{SalimansIS16}.

Another interesting aspect is that like most existing divergences comparing probability distributions, the proposed approach is based on likelihood ratios that only compare samples having the same values.
More flexible ways do exist to compare distributions, based for instance on optimal transport, such as the Wasserstein distance (e.g 1-Wasserstein GAN \cite{arjovsky2017wasserstein}) and could be adapted to keep the notion of trade-off between quality and diversity.

\appendix

\section{Proof of Lemma~\ref{lemma-NP}}
\label{app:lemma-NP}
Let $\varepsilon\geq 0$ such that $\Pr(U'=1|U=0)=\Pr(\tilde U=1|U=0)-\varepsilon$.
First, we decompose $\beta_\lambda'$ into 4 terms
\begin{equation*}
    \begin{split}
        \beta_\lambda'=&\Pr(U'=0|U=1)+\tfrac 1\lambda \Pr(U'=1|U=0)\\
        =& \Pr(U'=0, \tilde U=0| U=1) + \Pr(U'=0,\tilde U=1 | U=1) \\
        & + \tfrac 1\lambda \Pr(U'=1|U=0)\\
        =& \Pr(\tilde U=0|U=1) - \Pr(\tilde U=0, U'=1|U=1) 
        \\
        &+ \Pr(U'=0,\tilde U=1 | U=1) + \tfrac 1\lambda \Pr(U'=1|U=0).
    \end{split}
\end{equation*}
Considering separately each of the previous terms, we have 
\begin{equation*}
    \begin{split}
        A &= \Pr(\tilde U=0|U=1) \,,
        \\
       - B=
       & \Pr(U'=1, \tilde U=0|U=1)=\int \1_{U'=1} \1_{\tilde U=0} dP\\
       \leq& \int \1_{U'=1} \1_{\tilde U=0} \tfrac 1\lambda dQ 
       = \tfrac 1\lambda \Pr( U'=1, \tilde U=0|U-0) \\
       =& \tfrac 1\lambda (\Pr(U'=1|U=0) - \Pr(U'=1,\tilde U=1|U=0))\\
       =& \tfrac 1\lambda \Big(\Pr(\tilde U=1|U=0)-\varepsilon\\ &{\color{white}\tfrac 1\lambda}- (\Pr(\tilde U=1|U=0) - \Pr(U'=0, \tilde U=1|U=0))\Big)\\
       =& \tfrac 1\lambda( \Pr(U'=0, \tilde U=1|U=0)-\varepsilon).
    \end{split}
\end{equation*}
Finally 
\begin{equation*}
    \begin{split}
        B\geq& - \tfrac 1\lambda( \Pr(U'=0, \tilde U=1|U=0)-\varepsilon).
    \end{split}
\end{equation*}
Similarly
\begin{equation*}
    \begin{split}
        C =& \Pr(U'=0,\tilde U=1 | U=1) = \int \1_{U'=0}\1_{\tilde U=1}dP\\
        \geq& \int \1_{U'=0} \1_{\tilde U=1} \tfrac 1\lambda dQ 
        = \tfrac 1\lambda \Pr(U'=0, \tilde U=1|U=0) \,.
    \end{split}
\end{equation*}
Using both inequalities for $B$ and $C$, one gets
\begin{equation*}
    B+C \geq \frac \varepsilon\lambda . 
\end{equation*}
Last, 
\begin{equation*}
    \begin{split}
        D =&\tfrac 1\lambda \Pr(U'=1|U=0) = \tfrac 1\lambda (\Pr(\tilde U=1|U=0)-\varepsilon).\\
    \end{split}
\end{equation*}
Putting everything together, namely $\beta_\lambda' = A+B+C+D$, yields
\begin{equation*}
    \begin{split}
        \beta_\lambda'\geq& \Pr(\tilde U=0|U=1)+\tfrac 1\lambda \Pr(\tilde U=1|U=0)=\beta_\lambda.
    \end{split}
\end{equation*}
Using (by a slight abuse of notation) $\alpha_\lambda = \lambda \beta_\lambda$ and $\alpha_\lambda' = \lambda \beta_\lambda \geq \alpha_\lambda$ concludes the proof.
\qed{}

\clearpage
\section*{Acknowledgments}
This work was supported by fundings from {\em R\'egion Normandie} under grant \emph{RIN NormanD'eep}. 
The authors are grateful to the anonymous reviewers for their valuable comments and suggestions.

\bibliography{references}
\bibliographystyle{icml2019}

\end{document}